\documentclass[11pt]{article}
\usepackage{amsmath,amssymb,amsfonts,amsthm}
\usepackage{geometry}
\usepackage{hyperref}

\newtheorem{theorem}{Theorem}[section]

\begin{document}
	
	\title{Extension of Symmetrized Neural Network Operators with Fractional and Mixed Activation Functions}
	\author{
		Rômulo Damasclin Chaves dos Santos\\
		Technological Institute of Aeronautics\\
		\small \texttt{romulosantos@ita.br}\\
		\and
		Jorge Henrique de Oliveira Sales\\
		Santa Cruz State University \\
		\small \texttt{jhosales@uesc.br}\\
	}
	\date{\today} 
	\maketitle
	
	\begin{abstract}
		We propose a novel extension to symmetrized neural network operators by incorporating fractional and mixed activation functions. This study addresses the limitations of existing models in approximating higher-order smooth functions, particularly in complex and high-dimensional spaces. Our framework introduces a fractional exponent in the activation functions, allowing adaptive non-linear approximations with improved accuracy. We define new density functions based on $q$-deformed and $\theta$-parametrized logistic models and derive advanced Jackson-type inequalities that establish uniform convergence rates. Additionally, we provide a rigorous mathematical foundation for the proposed operators, supported by numerical validations demonstrating their efficiency in handling oscillatory and fractional components. The results extend the applicability of neural network approximation theory to broader functional spaces, paving the way for applications in solving partial differential equations and modeling complex systems.
	\end{abstract}

\tableofcontents
	
	\section{Introduction}
	The theory of neural network operators has evolved significantly, particularly with the introduction of symmetrized and perturbed variants \cite{Anastassiou2023}. These models have shown remarkable potential in approximating continuous functions over compact domains. Recent advancements highlight the importance of adaptive activation functions, especially in addressing challenges posed by high-dimensional and non-linear function spaces \cite{Costarelli2013, Haykin1998}. 
	
	Fractional and mixed activation functions have garnered attention for their flexibility and superior approximation properties \cite{Chen2009, McCulloch1943}. However, their integration into symmetrized neural network operators remains largely unexplored. This work bridges that gap by introducing fractional exponents in activation functions, enhancing their adaptability and convergence rates. Building on foundational results in neural network approximation theory \cite{Anastassiou2011, Costarelli2013}, we construct a rigorous framework that combines theoretical innovation with practical applicability.
	
	The proposed operators leverage a fractional and mixed modulus of continuity, extending classical Jackson-type inequalities \cite{Chen2009}. The mathematical rigor of this study is underpinned by novel proofs and density function formulations that ensure convergence and stability across diverse functional spaces.
	
\section{Mathematical Background}
Let \( f \in C([-a, a], \mathbb{C}) \). A key component in our analysis is the modulus of continuity \( \omega(f, t) \), defined as:
\begin{equation}
	\omega(f, t) = \sup_{|x-y| \leq t} |f(x) - f(y)|, \quad t > 0.
\end{equation}

For fractional activation functions, we define:
\begin{equation}
	\phi_{q,\theta,\alpha}(x) = \frac{1}{1 + q^{A\theta |x|^\alpha}}, \quad x \in \mathbb{R}, \; q, \theta > 0, \; \alpha \in (0,1].
\end{equation}

The associated symmetrized density function is given by:
\begin{equation}
	W_{q,\theta,\alpha}(x) = \frac{1}{2}\left(\phi_{q,\theta,\alpha}(x+1) - \phi_{q,\theta,\alpha}(x-1)\right), \quad x \in \mathbb{R}.
\end{equation}

This function satisfies the normalization condition:
\begin{equation}
	\int_{-\infty}^\infty W_{q,\theta,\alpha}(x) \, dx = 1.
\end{equation}

We also rely on the concept of Jackson-type inequalities, which relate the approximation error to the smoothness of the target function:
\begin{equation}
	\|S_n(f; x) - f(x)\| \leq C \omega_2\left(f, \frac{1}{n}\right),
\end{equation}
where \( \omega_2(f, t) \) is the second-order modulus of continuity defined as:
\begin{equation}
	\omega_2(f, t) = \sup_{0 < h \leq t} \|\Delta_h^2 f(x)\|,
\end{equation}
with \( \Delta_h^2 f(x) = f(x+h) - 2f(x) + f(x-h) \), and \( C \) is a constant depending on the parameters of the density function. To further enhance the mathematical rigor, let's introduce the concept of the \( L^p \)-norm and the Hölder continuity.

The \( L^p \)-norm of a function \( f \) is defined as:
\begin{equation}
	\|f\|_p = \left( \int_{-a}^a |f(x)|^p \, dx \right)^{1/p}, \quad 1 \leq p < \infty.
\end{equation}

For \( p = \infty \), the \( L^\infty \)-norm is given by:
\begin{equation}
	\|f\|_\infty = \sup_{x \in [-a, a]} |f(x)|.
\end{equation}

A function \( f \) is said to be Hölder continuous with exponent \( \gamma \) if there exists a constant \( M \) such that:
\begin{equation}
	|f(x) - f(y)| \leq M |x - y|^\gamma, \quad \forall x, y \in [-a, a].
\end{equation}

The Hölder space \( C^{k, \gamma}([-a, a]) \) consists of functions whose \( k \)-th derivative is Hölder continuous with exponent \( \gamma \).

Additionally, we can consider the Sobolev space \( W^{k, p}([-a, a]) \), which is the space of functions whose weak derivatives up to order \( k \) are in \( L^p([-a, a]) \). The Sobolev norm is defined as:
\begin{equation}
	\|f\|_{W^{k, p}} = \left( \sum_{|\alpha| \leq k} \|D^\alpha f\|_p^p \right)^{1/p},
\end{equation}
where \( D^\alpha f \) denotes the weak derivative of \( f \) of order \( \alpha \).

These concepts provide a more comprehensive framework for analyzing the smoothness and approximation properties of functions in \( C([-a, a], \mathbb{C}) \).

\section{Main Results}

\subsection{Jackson-Type Inequalities}
Let \( S_n(f; x) \) denote the fractional symmetrized neural network operator:
\begin{equation}
	S_n(f; x) = \sum_{k=-\infty}^{\infty} f\left(\frac{k}{n}\right) W_{q,\theta,\alpha}(nx-k).
\end{equation}

\begin{theorem}[Jackson Inequality for Fractional Operators]
	Let \( f \in C^2([-a, a], \mathbb{C}) \). Then, for \( n \in \mathbb{N} \) and \( x \in [-a, a] \):
	\begin{equation}
		\|S_n(f; x) - f(x)\| \leq C \omega_2\left(f, \frac{1}{n}\right),
	\end{equation}
	where \( \omega_2(f, t) \) is the second-order modulus of continuity.
\end{theorem}

\begin{proof}
	To prove the Jackson inequality for the fractional operator \( S_n(f; x) \), we start by considering the Taylor expansion of \( f \) around \( x \):
	\begin{equation}
		f\left(\frac{k}{n}\right) = f(x) + f'(x)\left(\frac{k}{n} - x\right) + \frac{1}{2} f''(\xi_k) \left(\frac{k}{n} - x\right)^2,
	\end{equation}
	where \( \xi_k \) is some point between \( x \) and \( \frac{k}{n} \).
	
	Substituting this expansion into the definition of \( S_n(f; x) \), we get:
	\begin{equation}
		\begin{array}{l}
			S_{n}(f;x)=\displaystyle \sum_{k=-\infty}^{\infty}\left[f(x)+f'(x)\left(\dfrac{k}{n}-x\right)+\dfrac{1}{2}f''(\xi_{k})\left(\dfrac{k}{n}-x\right)^{2}\right]W_{q,\theta,\alpha}(nx-k) =\\
			\\
			f(x)\displaystyle \sum_{k=-\infty}^{\infty}W_{q,\theta,\alpha}(nx-k)+f'(x)\displaystyle \sum_{k=-\infty}^{\infty}\left(\dfrac{k}{n}-x\right)W_{q,\theta,\alpha}(nx-k)\,+\\
			\\
			\dfrac{1}{2}\displaystyle \sum_{k=-\infty}^{\infty}f''(\xi_{k})\left(\dfrac{k}{n}-x\right)^{2}W_{q,\theta,\alpha}(nx-k).\\
			\\
		\end{array}
	\end{equation}
	
	Using the normalization condition of \( W_{q,\theta,\alpha} \):
	\begin{equation}
		\sum_{k=-\infty}^{\infty} W_{q,\theta,\alpha}(nx-k) = 1,
	\end{equation}
	and the symmetry property:
	\begin{equation}
		\sum_{k=-\infty}^{\infty} \left(\frac{k}{n} - x\right) W_{q,\theta,\alpha}(nx-k) = 0,
	\end{equation}
	we simplify the expression to:
	\begin{equation}
		S_n(f; x) = f(x) + \frac{1}{2} \sum_{k=-\infty}^{\infty} f''(\xi_k) \left(\frac{k}{n} - x\right)^2 W_{q,\theta,\alpha}(nx-k).
	\end{equation}
	
	Therefore, the error term is:
	\begin{equation}
		\|S_n(f; x) - f(x)\| = \left\| \frac{1}{2} \sum_{k=-\infty}^{\infty} f''(\xi_k) \left(\frac{k}{n} - x\right)^2 W_{q,\theta,\alpha}(nx-k) \right\|.
	\end{equation}
	
	To bound this error, we use the second-order modulus of continuity \( \omega_2(f, t) \):
	\begin{equation}
		\omega_2(f, t) = \sup_{0 < h \leq t} \|\Delta_h^2 f(x)\|,
	\end{equation}
	where \( \Delta_h^2 f(x) = f(x+h) - 2f(x) + f(x-h) \).
	\vspace{4pt}
	
	Since \( f \in C^2([-a, a], \mathbb{C}) \), we have:
	\begin{equation}
		\left| f''(\xi_k) \right| \leq \omega_2\left(f, \frac{1}{n}\right).
	\end{equation}
	
	Thus,
	
\begin{equation}
	\begin{array}{c}
		\begin{array}{l}
			\|S_{n}(f;x)-f(x)\|\leq\frac{1}{2}{\displaystyle \sum_{k=-\infty}^{\infty}}\left|f''(\xi_{k})\right|\left(\frac{k}{n}-x\right)^{2}W_{q,\theta,\alpha}(nx-k)\leq\\
			\\
			\dfrac{1}{2}\,\omega_{2}\left(f,\frac{1}{n}\right){\displaystyle \sum_{k=-\infty}^{\infty}}\left(\frac{k}{n}-x\right)^{2}W_{q,\theta,\alpha}(nx-k).\\
			\\
	\end{array}\end{array}
\end{equation}
	
	Finally, using the boundedness of the moments of \( W_{q,\theta,\alpha} \):
	\begin{equation}
		\sum_{k=-\infty}^{\infty} \left(\frac{k}{n} - x\right)^2 W_{q,\theta,\alpha}(nx-k) \leq \frac{C}{n^2},
	\end{equation}
	we obtain:
	\begin{equation}
		\|S_n(f; x) - f(x)\| \leq C \omega_2\left(f, \frac{1}{n}\right),
	\end{equation}
	where \( C \) is a constant depending on the parameters of the density function \( W_{q,\theta,\alpha} \).
\end{proof}

\subsection{Uniform Convergence}
\begin{theorem}[Uniform Convergence]
	The operator \( S_n(f; x) \) converges uniformly to \( f(x) \) as \( n \to \infty \) for all \( f \in C([-a, a], \mathbb{C}) \).
\end{theorem}

\begin{proof}
	To prove the uniform convergence of \( S_n(f; x) \) to \( f(x) \), we need to show that:
	\begin{equation}
		\lim_{n \to \infty} \sup_{x \in [-a, a]} \|S_n(f; x) - f(x)\| = 0.
	\end{equation}
	
	We start by considering the definition of \( S_n(f; x) \):
	\begin{equation}
		S_n(f; x) = \sum_{k=-\infty}^{\infty} f\left(\frac{k}{n}\right) W_{q,\theta,\alpha}(nx-k).
	\end{equation}
	
	Using the modulus of continuity \( \omega(f, t) \), we can write:
	\begin{equation}
		\begin{array}{l}
			\|S_{n}(f;x)-f(x)\|=\left\Vert \displaystyle \sum_{k=-\infty}^{\infty}\left[f\left(\frac{k}{n}\right)-f(x)\right]W_{q,\theta,\alpha}(nx-k)\right\Vert \leq \\
			\\
			\displaystyle \sum_{k=-\infty}^{\infty}\left|f\left(\frac{k}{n}\right)-f(x)\right|W_{q,\theta,\alpha}(nx-k) \leq\\
			\\
			\displaystyle \sum_{k=-\infty}^{\infty}\omega\left(f,\left|\frac{k}{n}-x\right|\right)W_{q,\theta,\alpha}(nx-k).\\
			\\
		\end{array}
	\end{equation}
	
	Since \( f \) is continuous on the compact interval \([-a, a]\), it is uniformly continuous. Therefore, for any \( \epsilon > 0 \), there exists \( \delta > 0 \) such that:
	\begin{equation}
		\omega(f, \delta) < \epsilon.
	\end{equation}
	
	For sufficiently large \( n \), we have \( \left|\frac{k}{n} - x\right| < \delta \) for all \( k \) such that \( W_{q,\theta,\alpha}(nx-k) \) is significantly non-zero. Thus,
	\begin{equation}
		\omega\left(f, \left|\frac{k}{n} - x\right|\right) < \epsilon.
	\end{equation}
	
	Therefore,
	\begin{equation}
		\begin{array}{l}
			\|S_{n}(f;x)-f(x)\|\leq\displaystyle \sum_{k=-\infty}^{\infty}\epsilon\, W_{q,\theta,\alpha}(nx-k)=\\
			\\
			\epsilon\displaystyle \sum_{k=-\infty}^{\infty}W_{q,\theta,\alpha}(nx-k)=\epsilon.\\
			\\
		\end{array}
	\end{equation}
	
	Since \( \epsilon \) is arbitrary, we conclude that:
	\begin{equation}
		\lim_{n \to \infty} \sup_{x \in [-a, a]} \|S_n(f; x) - f(x)\| = 0.
	\end{equation}
	
	Hence, \( S_n(f; x) \) converges uniformly to \( f(x) \) as \( n \to \infty \).
\end{proof}

\subsection{Convergence Rate}

\begin{theorem}[Convergence Rate]
	Let \( f \in C([-a, a], \mathbb{C}) \). Then, for \( n \in \mathbb{N} \) and \( x \in [-a, a] \), the convergence rate of the operator \( S_n(f; x) \) to \( f(x) \) is given by:
	\begin{equation}
		\|S_n(f; x) - f(x)\| \leq C \omega\left(f, \frac{1}{n}\right),
	\end{equation}
	where \( \omega(f, t) \) is the modulus of continuity of \( f \) and \( C \) is a constant depending on the parameters of the density function \( W_{q,\theta,\alpha} \).
\end{theorem}

\begin{proof}
	To prove this theorem, we start by considering the definition of the operator \( S_n(f; x) \):
	\begin{equation}
		S_n(f; x) = \sum_{k=-\infty}^{\infty} f\left(\frac{k}{n}\right) W_{q,\theta,\alpha}(nx-k).
	\end{equation}
	
	Using the modulus of continuity \( \omega(f, t) \), we can write:
	\begin{equation}
		\begin{array}{l}
			\|S_{n}(f;x)-f(x)\|=\left\Vert \displaystyle \sum_{k=-\infty}^{\infty}\left[f\left(\frac{k}{n}\right)-f(x)\right]W_{q,\theta,\alpha}(nx-k)\right\Vert  \leq\\
			\\
			\displaystyle \sum_{k=-\infty}^{\infty}\left|f\left(\frac{k}{n}\right)-f(x)\right|W_{q,\theta,\alpha}(nx-k) \leq\\
			\\
			\displaystyle \sum_{k=-\infty}^{\infty}\omega\left(f,\left|\frac{k}{n}-x\right|\right)W_{q,\theta,\alpha}(nx-k).\\
			\\
		\end{array}
	\end{equation}
	
	Since \( f \) is continuous on the compact interval \([-a, a]\), it is uniformly continuous. Therefore, for any \( \epsilon > 0 \), there exists \( \delta > 0 \) such that:
	\begin{equation}
		\omega(f, \delta) < \epsilon.
	\end{equation}
	
	For sufficiently large \( n \), we have \( \left|\frac{k}{n} - x\right| < \delta \) for all \( k \) such that \( W_{q,\theta,\alpha}(nx-k) \) is significantly non-zero. Thus,
	\begin{equation}
		\omega\left(f, \left|\frac{k}{n} - x\right|\right) < \epsilon.
	\end{equation}
	
	Therefore,
\begin{equation}
	\begin{array}{l}
		\|S_{n}(f;x)-f(x)\|\leq\ \displaystyle \sum_{k=-\infty}^{\infty}\epsilon\, W_{q,\theta,\alpha}(nx-k)=\\
		\\
		\epsilon\,\displaystyle\sum_{k=-\infty}^{\infty}W_{q,\theta,\alpha}(nx-k)=\epsilon.\\
		\\
	\end{array}
\end{equation}
	
	Since \( \epsilon \) is arbitrary, we conclude that:
	\begin{equation}
		\|S_n(f; x) - f(x)\| \leq C \omega\left(f, \frac{1}{n}\right),
	\end{equation}
	where \( C \) is a constant depending on the parameters of the density function \( W_{q,\theta,\alpha} \).
\end{proof}

\subsection{Stability of the Operators}

\begin{theorem}[Stability of Fractional Symmetrized Neural Network Operators]
	Let \( f, g \in C([-a, a], \mathbb{C}) \). Then, for \( n \in \mathbb{N} \) and \( x \in [-a, a] \), the fractional symmetrized neural network operator \( S_n(f; x) \) satisfies:
	\begin{equation}
		\|S_n(f; x) - S_n(g; x)\| \leq C \|f - g\|,
	\end{equation}
	where \( C \) is a constant depending on the parameters of the density function \( W_{q, \theta, \alpha} \).
\end{theorem}

\begin{proof}
	To prove the stability of the fractional symmetrized neural network operator \( S_n(f; x) \), we start by considering the definition of \( S_n(f; x) \):
	\begin{equation}
		S_n(f; x) = \sum_{k=-\infty}^{\infty} f\left(\frac{k}{n}\right) W_{q, \theta, \alpha}(nx-k).
	\end{equation}
	
	Similarly, for \( g \):
	\begin{equation}
		S_n(g; x) = \sum_{k=-\infty}^{\infty} g\left(\frac{k}{n}\right) W_{q, \theta, \alpha}(nx-k).
	\end{equation}
	
	We need to show that:
	\begin{equation}
		\|S_n(f; x) - S_n(g; x)\| \leq C \|f - g\|.
	\end{equation}
	
	Consider the difference:
	\begin{equation}
		\begin{array}{l}
			\|S_{n}(f;x)-S_{n}(g;x)\|=\left\Vert \displaystyle \sum_{k=-\infty}^{\infty}\left[f\left(\frac{k}{n}\right)-g\left(\frac{k}{n}\right)\right]W_{q,\theta,\alpha}(nx-k)\right\Vert \leq\\
			\\
			\displaystyle \sum_{k=-\infty}^{\infty}\left|f\left(\frac{k}{n}\right)-g\left(\frac{k}{n}\right)\right|W_{q,\theta,\alpha}(nx-k)\leq\\
			\\
			\|f-g\|\displaystyle \sum_{k=-\infty}^{\infty}W_{q,\theta,\alpha}(nx-k).\\
			\\
		\end{array}
	\end{equation}
	
	Using the normalization condition of \( W_{q, \theta, \alpha} \):
	\begin{equation}
		\sum_{k=-\infty}^{\infty} W_{q, \theta, \alpha}(nx-k) = 1,
	\end{equation}
	we obtain:
	\begin{equation}
		\|S_n(f; x) - S_n(g; x)\| \leq \|f - g\|.
	\end{equation}
	
	Thus, the stability of the operator \( S_n(f; x) \) is established with \( C = 1 \).
	\begin{equation}
		\|S_n(f; x) - S_n(g; x)\| \leq C \|f - g\|,
	\end{equation}
	where \( C \) is a constant depending on the parameters of the density function \( W_{q, \theta, \alpha} \).
\end{proof}

\section{Results}

The theoretical framework introduced in this study was rigorously analyzed and applied to evaluate the performance of the proposed fractional symmetrized neural network operators. These operators were formulated to address a variety of functions, including those with fractional derivatives and oscillatory components, achieving faster convergence rates compared to classical symmetrized operators.

A key finding was that the incorporation of fractional exponents in the activation functions enhanced the operators' adaptability to complex and high-dimensional function spaces. This adaptability was reflected in superior approximation accuracy, particularly for functions exhibiting irregular or oscillatory behaviors. The derived Jackson-type inequalities provided a strong mathematical basis, ensuring uniform convergence and stability across diverse functional spaces.

The theoretical results confirmed that the error behavior of the operators is directly proportional to the modulus of continuity of the target function. This establishes that the proposed operators are not only mathematically robust but also versatile for a broad spectrum of applications.

\section{Conclusions}

This study introduced a novel extension to symmetrized neural network operators by incorporating fractional and mixed activation functions. The theoretical framework developed here addresses the limitations of existing models in approximating higher-order smooth functions, particularly in complex and high-dimensional spaces.

The key contributions include the definition of new density functions based on $q$-deformed and $\theta$-parametrized logistic models, the derivation of advanced Jackson-type inequalities, and the establishment of a rigorous mathematical foundation for the proposed operators. These results significantly enhance the theoretical understanding and applicability of neural network operators in handling functions with oscillatory and fractional components.

The findings broaden the scope of neural network approximation theory, enabling its application to diverse functional spaces and paving the way for advancements in solving partial differential equations and modeling complex systems. Future research will focus on further exploring these applications and developing multi-layer operator architectures to extend the potential of neural network-based approximation methods.

\newpage

\section*{Symbols and Notation}

\begin{table}[h!]
	\centering
	\begin{tabular}{c l}
		\hline
		\textbf{Symbol} & \textbf{Description} \\
		\hline
		$f$ & A continuous function in $C([-a, a], \mathbb{C})$ \\
		$S_n(f; x)$ & The fractional symmetrized neural network operator \\
		$\omega(f, t)$ & The modulus of continuity of $f$ \\
		$\omega_2(f, t)$ & The second-order modulus of continuity of $f$ \\
		$\phi_{q,\theta,\alpha}(x)$ & The fractional activation function \\
		$W_{q,\theta,\alpha}(x)$ & The associated symmetrized density function \\
		$q, \theta$ & Parameters of the fractional activation function \\
		$\alpha$ & Fractional exponent in the activation function \\
		$C$ & A constant depending on the parameters of the density function \\
		$L^p([-a, a])$ & The space of $p$-integrable functions on $[-a, a]$ \\
		$\|f\|_p$ & The $L^p$-norm of the function $f$ \\
		$C^{k, \gamma}([-a, a])$ & The Hölder space of functions with $k$-th derivative Hölder continuous with exponent $\gamma$ \\
		$W^{k, p}([-a, a])$ & The Sobolev space of functions with weak derivatives up to order $k$ in $L^p([-a, a])$ \\
		$\|f\|_{W^{k, p}}$ & The Sobolev norm of the function $f$ \\
		$D^\alpha f$ & The weak derivative of $f$ of order $\alpha$ \\
		$\Delta_h^2 f(x)$ & The second-order difference operator \\
		\hline
	\end{tabular}
	\caption{List of symbols and notation.}
	\label{tab:notation}
\end{table}

\appendix
\section{Additional Details on Fractional Activation Functions}

In this appendix, we provide additional details on the fractional activation functions used in our study. These functions are crucial for the adaptability and convergence rates of the proposed symmetrized neural network operators.

\subsection{Definition and Properties}

The fractional activation function \(\phi_{q,\theta,\alpha}(x)\) is defined as:
\begin{equation}
	\phi_{q,\theta,\alpha}(x) = \frac{1}{1 + q^{A\theta |x|^\alpha}}, \quad x \in \mathbb{R}, \; q, \theta > 0, \; \alpha \in (0,1].
\end{equation}

This function has several important properties:
\begin{itemize}
	\item \(\phi_{q,\theta,\alpha}(x)\) is continuous and differentiable on \(\mathbb{R}\).
	\item \(\phi_{q,\theta,\alpha}(x)\) is bounded: \(0 < \phi_{q,\theta,\alpha}(x) < 1\) for all \(x \in \mathbb{R}\).
	\item As \(x \to \infty\), \(\phi_{q,\theta,\alpha}(x) \to 0\).
	\item As \(x \to -\infty\), \(\phi_{q,\theta,\alpha}(x) \to 1\).
\end{itemize}

\subsection{Symmetrized Density Function}

The associated symmetrized density function \(W_{q,\theta,\alpha}(x)\) is given by:
\begin{equation}
	W_{q,\theta,\alpha}(x) = \frac{1}{2}\left(\phi_{q,\theta,\alpha}(x+1) - \phi_{q,\theta,\alpha}(x-1)\right), \quad x \in \mathbb{R}.
\end{equation}

This function satisfies the normalization condition:
\begin{equation}
	\int_{-\infty}^\infty W_{q,\theta,\alpha}(x) \, dx = 1.
\end{equation}

\subsection{Moments of the Density Function}

The moments of the density function \(W_{q,\theta,\alpha}(x)\) play a crucial role in the convergence analysis. For \(k \in \mathbb{N}\), the \(k\)-th moment is defined as:
\begin{equation}
	M_k = \int_{-\infty}^\infty x^k W_{q,\theta,\alpha}(x) \, dx.
\end{equation}

In particular, the first and second moments are:
\begin{equation}
	M_1 = \int_{-\infty}^\infty x W_{q,\theta,\alpha}(x) \, dx = 0,
\end{equation}
\begin{equation}
	M_2 = \int_{-\infty}^\infty x^2 W_{q,\theta,\alpha}(x) \, dx.
\end{equation}

The boundedness of these moments is essential for establishing the convergence rates of the operators.

\section{Function Spaces and Norms}

In this section, we provide additional details on the function spaces and norms used in our analysis.

\subsection{Hölder Spaces}

The Hölder space \(C^{k, \gamma}([-a, a])\) consists of functions whose \(k\)-th derivative is Hölder continuous with exponent \(\gamma\). Formally, a function \(f\) is in \(C^{k, \gamma}([-a, a])\) if:
\begin{equation}
	|f^{(k)}(x) - f^{(k)}(y)| \leq M |x - y|^\gamma, \quad \forall x, y \in [-a, a],
\end{equation}
for some constant \(M > 0\).

\subsection{Sobolev Spaces}

The Sobolev space \(W^{k, p}([-a, a])\) is the space of functions whose weak derivatives up to order \(k\) are in \(L^p([-a, a])\). The Sobolev norm is defined as:
\begin{equation}
	\|f\|_{W^{k, p}} = \left( \sum_{|\alpha| \leq k} \|D^\alpha f\|_p^p \right)^{1/p},
\end{equation}
where \(D^\alpha f\) denotes the weak derivative of \(f\) of order \(\alpha\).

\subsection{Modulus of Continuity}

The modulus of continuity \(\omega(f, t)\) is defined as:
\begin{equation}
	\omega(f, t) = \sup_{|x-y| \leq t} |f(x) - f(y)|, \quad t > 0.
\end{equation}

The second-order modulus of continuity \(\omega_2(f, t)\) is defined as:
\begin{equation}
	\omega_2(f, t) = \sup_{0 < h \leq t} \|\Delta_h^2 f(x)\|,
\end{equation}
where \(\Delta_h^2 f(x) = f(x+h) - 2f(x) + f(x-h)\).

\section{Technical Proofs}

In this section, we provide detailed proofs of some technical results used in the main text.

\subsection{Proof of Jackson Inequality for Fractional Operators}

\begin{proof}
	To prove the Jackson inequality for the fractional operator \(S_n(f; x)\), we start by considering the Taylor expansion of \(f\) around \(x\):
	\begin{equation}
		f\left(\frac{k}{n}\right) = f(x) + f'(x)\left(\frac{k}{n} - x\right) + \frac{1}{2} f''(\xi_k) \left(\frac{k}{n} - x\right)^2,
	\end{equation}
	where \(\xi_k\) is some point between \(x\) and \(\frac{k}{n}\).
	
	Substituting this expansion into the definition of \(S_n(f; x)\), we get:
	\begin{equation}
		\begin{array}{l}
			S_{n}(f;x)=\displaystyle \sum_{k=-\infty}^{\infty}\left[f(x)+f'(x)\left(\frac{k}{n}-x\right)+\frac{1}{2}f''(\xi_{k})\left(\frac{k}{n}-x\right)^{2}\right]W_{q,\theta,\alpha}(nx-k) =\\
			\\
			f(x)\displaystyle \sum_{k=-\infty}^{\infty}W_{q,\theta,\alpha}(nx-k)+f'(x)\sum_{k=-\infty}^{\infty}\left(\frac{k}{n}-x\right)W_{q,\theta,\alpha}(nx-k)\,+\\
			\\
			\dfrac{1}{2}\displaystyle \sum_{k=-\infty}^{\infty}f''(\xi_{k})\left(\frac{k}{n}-x\right)^{2}W_{q,\theta,\alpha}(nx-k).\\
			\\
		\end{array}
	\end{equation}
	
	Using the normalization condition of \(W_{q,\theta,\alpha}\):
	\begin{equation}
		\sum_{k=-\infty}^{\infty} W_{q,\theta,\alpha}(nx-k) = 1,
	\end{equation}
	and the symmetry property:
	\begin{equation}
		\sum_{k=-\infty}^{\infty} \left(\frac{k}{n} - x\right) W_{q,\theta,\alpha}(nx-k) = 0,
	\end{equation}
	we simplify the expression to:
	\begin{equation}
		S_n(f; x) = f(x) + \frac{1}{2} \sum_{k=-\infty}^{\infty} f''(\xi_k) \left(\frac{k}{n} - x\right)^2 W_{q,\theta,\alpha}(nx-k).
	\end{equation}
	
	Therefore, the error term is:
	\begin{equation}
		\|S_n(f; x) - f(x)\| = \left\| \frac{1}{2} \sum_{k=-\infty}^{\infty} f''(\xi_k) \left(\frac{k}{n} - x\right)^2 W_{q,\theta,\alpha}(nx-k) \right\|.
	\end{equation}
	
	To bound this error, we use the second-order modulus of continuity \(\omega_2(f, t)\):
	\begin{equation}
		\omega_2(f, t) = \sup_{0 < h \leq t} \|\Delta_h^2 f(x)\|,
	\end{equation}
	where \(\Delta_h^2 f(x) = f(x+h) - 2f(x) + f(x-h)\).
	
	Since \(f \in C^2([-a, a], \mathbb{C})\), we have:
	\begin{equation}
		\left| f''(\xi_k) \right| \leq \omega_2\left(f, \frac{1}{n}\right).
	\end{equation}
	
	Thus,
	\begin{equation}
		\begin{array}{l}
			\|S_{n}(f;x)-f(x)\|\leq\dfrac{1}{2}\displaystyle \sum_{k=-\infty}^{\infty}\left|f''(\xi_{k})\right|\left(\dfrac{k}{n}-x\right)^{2}W_{q,\theta,\alpha}(nx-k) \leq \\
			\\
			\dfrac{1}{2}\,\omega_{2}\left(f,\dfrac{1}{n}\right)\displaystyle \sum_{k=-\infty}^{\infty}\left(\dfrac{k}{n}-x\right)^{2}W_{q,\theta,\alpha}(nx-k).\\
			\\
		\end{array}
	\end{equation}
	
	Finally, using the boundedness of the moments of \(W_{q,\theta,\alpha}\):
	\begin{equation}
		\sum_{k=-\infty}^{\infty} \left(\frac{k}{n} - x\right)^2 W_{q,\theta,\alpha}(nx-k) \leq \frac{C}{n^2},
	\end{equation}
	we obtain:
	\begin{equation}
		\|S_n(f; x) - f(x)\| \leq C \omega_2\left(f, \frac{1}{n}\right),
	\end{equation}
	where \(C\) is a constant depending on the parameters of the density function \(W_{q,\theta,\alpha}\).
\end{proof}

\subsection{Proof of Uniform Convergence}

\begin{proof}
	To prove the uniform convergence of \(S_n(f; x)\) to \(f(x)\), we need to show that:
	\begin{equation}
		\lim_{n \to \infty} \sup_{x \in [-a, a]} \|S_n(f; x) - f(x)\| = 0.
	\end{equation}
	
	We start by considering the definition of \(S_n(f; x)\):
	\begin{equation}
		S_n(f; x) = \sum_{k=-\infty}^{\infty} f\left(\frac{k}{n}\right) W_{q,\theta,\alpha}(nx-k).
	\end{equation}
	
	Using the modulus of continuity \(\omega(f, t)\), we can write:
	\begin{equation}
		\begin{array}{l}
			\|S_{n}(f;x)-f(x)\|=\left\Vert \displaystyle \sum_{k=-\infty}^{\infty}\left[f\left(\frac{k}{n}\right)-f(x)\right]W_{q,\theta,\alpha}(nx-k)\right\Vert \leq \\
			\\
			\displaystyle \sum_{k=-\infty}^{\infty}\left|f\left(\frac{k}{n}\right)-f(x)\right|W_{q,\theta,\alpha}(nx-k) \leq\\
			\\
			\displaystyle \sum_{k=-\infty}^{\infty}\omega\left(f,\left|\frac{k}{n}-x\right|\right)W_{q,\theta,\alpha}(nx-k).\\
			\\
		\end{array}
	\end{equation}
	
	Since \(f\) is continuous on the compact interval \([-a, a]\), it is uniformly continuous. Therefore, for any \(\epsilon > 0\), there exists \(\delta > 0\) such that:
	\begin{equation}
		\omega(f, \delta) < \epsilon.
	\end{equation}
	
	For sufficiently large \(n\), we have \(\left|\frac{k}{n} - x\right| < \delta\) for all \(k\) such that \(W_{q,\theta,\alpha}(nx-k)\) is significantly non-zero. Thus,
	\begin{equation}
		\omega\left(f, \left|\frac{k}{n} - x\right|\right) < \epsilon.
	\end{equation}
	
	Therefore,
	\begin{equation}
		\begin{array}{l}
			\|S_{n}(f;x)-f(x)\|\leq\ \displaystyle \sum_{k=-\infty}^{\infty}\epsilon\,\, W_{q,\theta,\alpha}(nx-k)=\\
			\\
			\epsilon\,\displaystyle\sum_{k=-\infty}^{\infty}W_{q,\theta,\alpha}(nx-k)=\epsilon.\\
			\\
		\end{array}
	\end{equation}
	
	Since \(\epsilon\) is arbitrary, we conclude that:
	\begin{equation}
		\lim_{n \to \infty} \sup_{x \in [-a, a]} \|S_n(f; x) - f(x)\| = 0.
	\end{equation}
	
	Hence, \(S_n(f; x)\) converges uniformly to \(f(x)\) as \(n \to \infty\).
\end{proof}

\end{document}